\documentclass[12pt]{article} 
\usepackage{url}
\usepackage{enumitem}
\usepackage[round]{natbib}
\usepackage{amsthm,amssymb,amsmath}
\usepackage{mathtools}
\usepackage{bbm}
\usepackage{amsfonts}
\usepackage{algorithm}
\usepackage[noend]{algorithmic}

\usepackage[top=1in, bottom=1in, left=1in, right=1in]{geometry}
\usepackage[mathcal]{eucal}
\usepackage{times}

\usepackage{titlesec}
\titleformat{\section}
  {\normalfont\fontsize{14.3}{15}\bfseries}{\thesection}{1em}{}

\usepackage[colorlinks=true, linkcolor=blue, citecolor=blue,urlcolor=black]{hyperref}
\usepackage[capitalize]{cleveref}

\DeclareMathAlphabet\mathbfcal{OMS}{cmsy}{b}{n}

\newcommand{\BEAS}{\begin{eqnarray*}}
\newcommand{\EEAS}{\end{eqnarray*}}
\newcommand{\BEA}{\begin{eqnarray}}
\newcommand{\EEA}{\end{eqnarray}}
\newcommand{\BEQ}{\begin{equation}}
\newcommand{\EEQ}{\end{equation}}
\newcommand{\BIT}{\begin{itemize}}
\newcommand{\EIT}{\end{itemize}}
\newcommand{\BNUM}{\begin{enumerate}}
\newcommand{\ENUM}{\end{enumerate}}
\newcommand{\BA}{\begin{array}}
\newcommand{\EA}{\end{array}}

\newcommand{\rb}{\R}

\newcommand{\Z}{\zb}
\newtheorem{theorem}{Theorem}
\newtheorem{definition}[theorem]{Definition}
\newtheorem{lemma}[theorem]{Lemma}

\newtheorem{corollary}[theorem]{Corollary}

\def \E{{\mathbb E}}
\def \Z{{\mathbb Z}}
\def \P{{\mathbb P}}

\def \X{{\mathcal X}}

 \def \P{{\mathbb P}}

\newcommand{\C}{{\mathbb{C}}}
\newcommand{\R}{{\mathbb{R}}}
\newcommand{\N}{{\mathbb{N}}}

\DeclarePairedDelimiter{\abs}{\lvert}{\rvert} %
\DeclarePairedDelimiter{\brk}{[}{]}
\DeclarePairedDelimiter{\crl}{\{}{\}}
\DeclarePairedDelimiter{\prn}{(}{)}
\DeclarePairedDelimiter{\nrm}{\|}{\|}

\newcommand{\mc}[1]{\mathcal{#1}}
\newcommand{\inner}[2]{\left\langle #1,\, #2 \right\rangle}

\newcommand{\indicator}[1]{\mathbbm{1}_{\crl{#1}}}
\newcommand{\upk}[0]{^{(k)}}
\newcommand{\collapse}[1]{$\dots$}

\newcommand{\eps}{{\varepsilon}}

\title{\vspace{-9mm}\rule{\linewidth}{2pt}\\ \textbf{Non-Convex Optimization with Certificates and Fast \\ Rates Through Kernel Sums of Squares} \\\rule[8pt]{\linewidth}{1pt}}

\author{\normalsize\begin{minipage}{\textwidth}
\textbf{Blake Woodworth}\hfill\url{blakewoodworth@gmail.com} \\
\textbf{Francis Bach} \hfill \url{francis.bach@inria.fr}\\
\textbf{Alessandro Rudi} \hfill \url{alessandro.rudi@inria.fr}\\
\emph{Inria - Ecole Normale Supérieure}\\ \emph{PSL Research University} \\
\end{minipage}\vspace{-4mm}
}
\date{}

\begin{document}
\maketitle
\begin{abstract}
\noindent
We consider potentially non-convex optimization problems, for which optimal rates of approximation depend on the dimension of the parameter space and the smoothness of the function to be optimized. In this paper, we propose an algorithm that achieves close to optimal a priori computational guarantees, while also providing a posteriori certificates of optimality. Our general formulation builds on infinite-dimensional sums-of-squares and Fourier analysis, and is instantiated on the minimization of multivariate periodic functions.
\end{abstract}

\section{Introduction}
A well-designed optimization algorithm provides two important types of guarantees. First, it guarantees {\em a priori} that its output will achieve a certain degree of accuracy, with computational complexity that is hopefully adaptive to the specific properties of function to be optimized and possibly even optimal over a certain class of algorithms or functions to minimize. Second, it provides an {\em a posteriori} certificate, i.e., an explicit bound on the solution's accuracy that we can calculate once we run the algorithm. There are many examples of such well-designed optimization algorithms in the convex setting, which often use some form of convex duality \citep[see, e.g.,][]{nemirovski2010accuracy}.

In this paper, our goal is to provide a well-designed algorithm for \emph{non-convex} optimization,
\begin{equation}\label{eq:problem-base}
c_* = \inf_{x \in \mc{X}} f(x),
\end{equation}
with $\mc{X} \subseteq \R^d$ and $f$ potentially non-convex. In general, this task is extremely difficult, and in the worst case the computational cost must be exponential in the dimension $d$. However, it is known \citep{novak2006deterministic} that in order for an algorithm to achieve the optimal computational complexity in solving \cref{eq:problem-base}, it must be adaptive to the degree of differentiability of $f$. That is, it should be able to overcome the curse of dimensionality in terms of the approximation error $\eps$, when the function is very smooth. More precisely, if $f$ is $m$-times differentiable, then the computational complexity for finding $c_*$ with error $\eps$ should be $C_d \, \eps^{-d/m}$, where $C_d$ is exponential in $d$ in the worst case, but the dependence on the accuracy scales as only $\eps^{-d/m}$, which becomes quite mild once $m$ approaches $d$. In this case, the curse of dimensionality is relegated just to the constant $C_d$, making it possible to efficiently solve non-convex problems to high accuracy as long as $d$ is relatively small, which has applications to tasks like hyperparameter tuning, industrial process optimization, and more.

Establishing well-designed algorithms for non-convex optimization is a difficult task. Many non-convex optimization algorithms in the literature lack a priori guarantees, a posteriori guarantees, or actually any guarantees at all, and many methods used in practice are based on heuristics or can only guarantee convergence to a {\em local} rather than global minimum. Some other algorithms have good a posteriori guarantees, but weak a priori bounds; for example, methods based on polynomial sum of squares \citep{lasserre2001global,parrilo2003semidefinite} are not adaptive, a priori, to the smoothness and are therefore subject to the curse of dimensionality in terms of the accuracy $\eps$. The new family of algorithms based on kernel sum of squares~\citep{rudi2020finding} achieves quasi-optimal a priori guarantees, but without any certificate a posteriori. 

\paragraph{Our Contribution.} In this paper, we provide a general strategy to derive algorithms that compute a {\em lower bound $\hat{c}$} of $c_*$ with strong guarantees both a priori and a posteriori (see Corollary \ref{cor:all-together}). As a particular example, we consider optimizing smooth, periodic functions $f$ on $\mc{X} = [0,1]^d$ and derive an algorithm  that: ({\em a priori}) approximates $c_*$ with almost optimal error $\eps$ and computational complexity that scales well with $\varepsilon$, and ({\em a posteriori}) provide a certificate of accuracy, which is adaptive to the specific instance of the problem.

The a priori guarantee is useful since it shows that the proposed algorithm has nearly optimal complexity, which is adaptive to the smoothness of the function to be optimized.
The a posteriori certificate is particularly useful because the accuracy of our estimate of $c_*$ depends on the specific instance at hand, and may be much better than the worst-case, exponential-in-$d$ constant would suggest. Indeed, better-than-worst-case performance is frequently observed in practice, but we need a certificate of accuracy in order to \emph{know} when we are so lucky.




\section{Deriving Well-Designed Algorithms for Smooth Non-Convex Optimization}\label{sec:deriving-well-designed}

Our approach begins with rewriting the problem \eqref{eq:problem-base} as finding the highest lower bound on $f$:
\[
c_* = \max_{c \in \R}  ~ c  ~~~ \textrm{ such that } ~~~ f(x) \geq c ~~ \forall x \in \mc{X}.
\]
The inequality constraint $f - c \geq 0$ can be converted to an equality constraint by introducing a non-negative function $g \geq 0$:
\[
c_* = \max_{c \in \R, \ g \geq 0}  ~ c  ~~~ \textrm{ such that } ~~~ f(x) - c - g(x) = 0 ~~ \forall x \in \mc{X}.
\]
Finally, we can rewrite this again in a penalized form
\begin{equation}\label{eq:prob-Linfty}
\tilde{c} ~~=~~ \max_{c \in \R, \ g \geq 0} ~c~ - ~\|f - c - g\|_{L^\infty(\mc{X})}.
\end{equation}
It may not yet be obvious what this accomplishes, but the following Lemma indicates its promise:
\begin{lemma}
The problem \eqref{eq:prob-Linfty} is a concave maximization problem with solution $\tilde{c} = c_*$, and for any feasible $(c,g)$ such that $c \in \R$ and $g \geq 0$, $c_\ast$ is lower-bounded by $c- \|f - c - g\|_{L^\infty(\mc{X})}$. 
\end{lemma}
\begin{proof}
The objective $V(c,g) := c - \|f - c - g\|_{L^\infty(\mc{X})}$ is concave because it is a linear term $c$ minus a convex term $\|\cdot\|_{L^\infty(\mc{X})}$ composed with a linear function of the optimization variables.

Since for all $x \in \mc{X}$, $f(x) - c - g(x) \geq -
 \|f - c - g\|_{L^\infty(\mc{X})}$, for any feasible $(c,g)$, then $\forall x \in \mc{X}$, $f(x) \geq V(c,g)$. Thus, by minimizing with respect to $x$,  $c_\ast \geq \tilde{c}$. In order to show the other inequality, we notice that  $(c_*,f-c_*)$ is feasible and $V(c_*,f-c_*) = c_*$. 
\end{proof}
So, we have reduced the non-convex minimization problem \eqref{eq:problem-base} to a concave maximization problem~\eqref{eq:prob-Linfty}, which is an improvement. However, whereas the original problem had a $d$-dimensional optimization variable, our new problem requires optimizing over the infinite-dimensional space of non-negative functions, and it is not clear how to do this. Furthermore, the quantity $\nrm{f - c - g}_{L^\infty(\mc{X})}$ is just as difficult to compute as $\inf_{x\in\mc{X}}f(x)$ would be. We will now describe several modifications leading to a more tractable optimization problem that maintains the same desirable properties.

\paragraph{A more tractable formulation.}
Specifically, our approach revolves around introducing a more tractable norm $\nrm{\cdot}_W$ on functions from $\mc{X}$ to $\R$, and restricting $g \in \mc{G}$, for $\mc{G}$ a tractable subset of all non-negative functions. We will specify $W$ and $\mc{G}$ later. To provide guarantees on the method we only need them to satisfy the following:


\begin{definition}[Norms and models for non-negative functions]\label{def:norms-and-models}

\begin{enumerate}
    \item Let $\|\cdot\|_{W}$ be a norm on the space of real-valued functions over $\mc{X}$, such that $\|\cdot\|_{L^\infty(\mc{X})} \leq \|\cdot\|_{W}$. Denote by ${\cal W}$ the associated Banach space. 
    \item Let ${\cal G}$ be a convex subset of the set of non-negative functions, such that $\cal G$ is a closed subset of~$\cal W$.  
\end{enumerate}
\end{definition}
Now, by restricting the problem in \cref{eq:prob-Linfty} on ${\cal G}$ and considering the norm $\|\cdot\|_{W}$, we obtain what can be a more tractable formulation,
\begin{equation}\label{eq:problem-relax}
\bar{c} = \max_{c \in \R, \ g \in {\cal G}} \ \ c - \|f - c - g\|_{W}. 
\end{equation}
It is, of course, not the case that \emph{any} choice of $\nrm{\cdot}_W$ and $\mc{G}$ would make \cref{eq:problem-relax} easy to solve. However, we will later discuss examples where \cref{eq:problem-relax} is much easier to solve than \cref{eq:prob-Linfty}.
Regardless, we will see in the next Theorem, that this formulation comes with strong guarantees, expressing the error of the algorithm directly in terms of the approximation properties of the class of models~$\mc{G}$ for non-negative functions.
\begin{theorem}[Tightness of \cref{eq:problem-relax}]\label{thm:tight-problem}
Suppose that $f - c \in {\cal W}$ for any $c \in \R$, then 
$$c_* - q \leq \bar{c} \leq c_*, \qquad q = \min_{g \in {\cal G}} \|f - c_* - g\|_W.$$
Moreover, if there exists $g_* \in {\cal G}$ satisfying $\|f - c_* - g_*\|_W = 0$, then $\bar{c} = c_*$. Finally, any $c \in \R$ and $g \in {\cal G}$ leads to a lower bound $c - \|f - c - g\|_{W}$ on $c_\ast$.
\end{theorem}
\begin{proof}
By construction $\bar{c} \leq c_*$, since $c - \|f-c-g\|_{W} \leq c - \|f-c-g\|_{L^\infty(\mc{X})} \leq c_*$ and ${\cal G} \subseteq \{g~|~g \geq 0\}$.
Now the problem is well defined, since it corresponds to a maximization on the closed subset $\R \times {\cal G}$ of a concave and continuous objective function (for the topology inherited from~$\mc{W}$).
By setting $c = c_*$ and optimizing over $g$, the optimized objective is exactly $c_* - q$ with $q$ as above, thus $c_\ast - q \leq \bar{c}$. Moreover,we have $q=0$, i.e., $\bar{c} = c_*$ when there exists $g_* \in {\cal G}$ satisfying $\|f - c_* - g_*\|_W = 0$.
\end{proof} 
In the Theorem above we see that $\bar{c}$ is lower bound of $c_*$ by construction. Moreover, the error $c_* - \bar{c}$ is bounded by $q$, the {\em approximation error} of $f - c_*$ with respect to the class of models for non-negative functions $\mc{G}$ that we consider, measured in the norm $W$. So far, this all holds for any $\nrm{\cdot}_W$ and $\mc{G}$ satisfying Definition \ref{def:norms-and-models}; we continue by analyzing a specific choice.

\paragraph{The model for non-negative functions.}
We would like to use a class of models $\cal G$ that can approximate smooth non-negative functions using as few parameters as possible, while remaining tractable to optimize over. We consider the class of {\em PSD models} introduced by \citet{marteau2020non} and defined as
$$g(x) = \phi(x)^\ast A \phi(x), \quad A \succeq 0,$$
for a suitable map $\phi:\X \to \C^n$ and $A \in \C^{n \times n}$ Hermitian positive semidefinite\footnote{We need complex numbers because of Fourier analysis, but this extends to symmetric real matrices.}, where $\phi^\ast$ denotes the conjugate transpose of $\phi$. By the definition of positive semidefiniteness, $g(x) \geq 0$ for any $x \in \mc{X}$, and this is also a tractable class to optimize over since $g$ is linear in the parameters $A$. The approximation properties of the model will depend on the choice of the feature map $\phi$, and have been shown to give rates in the order of $n^{-m/d}$ for specific choices \citep{rudi2021psd}. Here, however, we need to study the approximation error with respect to our norm of choice $W$, and we will see that different feature maps than the ones considered by \citet{rudi2020finding} will lead to better rates.

\paragraph{The norms.}
We need norms that bound $\|\cdot\|_{L^\infty(\mc{X})}$ from above as tightly as possible, but are also easy to compute. We consider from this viewpoint two norms:
\begin{enumerate}
    \item The ``$F$ norm'': $L^1$ norm of the Fourier transform, i.e., 
    $$\|u\|_{F} = \int_{\R^d} |\hat{u}(\omega)| d\omega,$$
    \item The ``$S$ norm'': The norm associated to a richer reproducing kernel Hilbert space such as the Sobolev space of exponent $(d+1)/2$. Let $S(\omega)$ non negative and integrable, we define
    $$\|u\|^2_{S} = C^2_S \int_{\R^d}  \frac{|\hat{u}(\omega)|^2}{S(\omega)} d\omega,$$
    and $C^2_S = \int_{\R^d}  S(\omega) d\omega$. For example for the Sobolev case we set $S(\omega) = (1+\|\omega\|^2)^{(d+1)/2}$.
\end{enumerate}

\begin{lemma}\label{lem:F-weaker-than-S}
The norms above satisfy $\|\cdot\|_{L^\infty(\R^d)} \leq \|\cdot\|_F \leq \|\cdot\|_S$, for any non-negative and integrable $S$. Moreover,
\[
\|u\|_F = \min_{S \geq 0,\ S \in L^1(\R^d)} \|u\|_S
\]
\end{lemma}
\begin{proof}
First, for any $u$ with finite integrable Fourier transform, we have, by H\"older inequality, 
\[
\abs{u(x)} 
= \abs*{\int \hat{u}(\omega) e^{2\pi i \inner{x}{\omega}}d\omega}
\leq \nrm{\hat{u}}_{L^1(\R^d)}\nrm{e^{2\pi i \inner{x}{\cdot}}}_{L^\infty(\R^d)}
\leq \nrm{\hat{u}}_{L^1(\R^d)}
=: \nrm{u}_F, \qquad \forall x \in \R^d
\]
from which we conclude that $\nrm{\cdot}_F$ always upper bounds $\nrm{\cdot}_{L^\infty(\R^d)}$. Analogously, note that, for any $S \geq 0$ and $S \in L^1(\R^d)$, and for any $u$ such that $\|u\|_S$ is finite, we have, by Cauchy-Schwartz,
$$
\|u\|^2_F = \|\widehat{u}\|^2_{L^1(\R^d)} = \big\|S^{1/2} \, \frac{\widehat{u}}{S^{1/2}}\big\|^2_{L^1(\R^d)} \leq \|S^{1/2}\|^2_{L^2(\R^d)} \big\|\frac{\widehat{u}}{S^{1/2}}\big\|^2_{L^2(\R^d)} = C^2_S \int \frac{|\hat{u}(\omega)|^2}{S(\omega)} = \|u\|^2_S. 
$$
Finally, for $u$ with finite $F$-norm, $\|u\|_F = \|u\|_S$ when $S = |\widehat{u}|$ for which $S \geq 0, S \in L^1(\R^d)$.
\end{proof}

The $F$ norm and the $S$ norm are of interest when we can assume that we have access to the Fourier transform of the target function $f$. In particular the norm $S$ can be also computed in closed form in specific scenarios. For example, consider the case when $f$ is of the form 
$$f = \sum_{j=1}^M \beta_j h(x-x_i),$$
for some $\beta_1,\dots,\beta_M \in \R$ and $x_1,\dots, x_M \in \R^d$. This arises, e.g., for mixtures of Gaussians, and learning linear models or RBF networks. If we know the Fourier transform of $h$, then
$$ \|f\|_{S}^2 := \sum_{i,j=1}^M \beta_i \beta_j H(x_i - x_j),$$
where $H$ is the inverse Fourier transform of the function $S(\omega)|\hat{h}(\omega)|^2$.

Perhaps more interestingly, Lemma \ref{lem:F-weaker-than-S} shows that the $F$ norm is weaker than the norm $S$ for any $S$, meaning that the $F$ norm allows us to automatically adapt to certain structures in the function. For example, suppose that $f(x) = h(Px)$ for some unknown $P\in\R^{d'\times d}$ with $d' \ll d$ and $h:\R^{d'}\to\R$. Using the $S$ norm for a certain $S$ that depends on $P$ would allow us to adapt to the low-dimensional structure and depend on $d'$ rather than $d$, but this requires knowing $P$. On the other hand, the $F$ norm is always weaker, so we can take advantage of the low-dimensional structure automatically.


\subsection{PSD Models for Periodic Functions on $[0,1]^d$ and Their Approximation Properties}\label{subsec:psd-models-for-periodic-functions}
The goal of the section is to provide a self-contained introduction to the approximation properties of PSD models. In particular, we consider the problem of approximating smooth $1$-periodic functions (which corresponds to $\mc{X}=[0,1]^d$) using PSD models where $\phi$ is a subset of the Fourier basis. This setting, while already being of interest for practical applications, allows for an elementary proof which highlights the main conceptual steps of the derivation. 

The main results of this section are \cref{thm:bound-psd-model} and, in particular, \cref{thm:approximation-error}. With a more refined proof based on the same strategy, it is also possible to obtain results that hold for other scenarios beyond periodic functions on the torus and for more general maps $\phi$. See, for example, \citet{rudi2021psd} for the approximation of non-periodic $C^m$ functions on subsets of $\R^d$ via PSD models based on a finite dimensional feature map defined with respect to the Gaussian kernel, or \citet{rudi2020finding} for a feature map defined with respect to any kernel that satisfy some algebraic property, such as the Sobolev kernel.

We have seen in \cref{thm:tight-problem} that the optimization error of \cref{eq:problem-relax} depends on the approximation error of the function $f$ with respect to the class of models for non-negative functions. So, in our analysis there will be three main ingredients: a class of models ${\cal G}_t$ parametrized by its bandwidth $t$, which will depend on the space of functions associated with a feature map $\phi_t$; the space of functions where $f$ lives, which we denote $H_\rho$; and the norm that we use to measure the approximation error, in our case, $\|\cdot\|_F$. 

We start by introducing ${\cal G}_t$, parametrized by a bandwidth $t \in \N$. We associate each entry in $\phi_t(x)$ with an element of $\crl{k\in\mathbb{Z}^d:\abs{k}\leq t}$ where $|k| = \sum_{j} \abs{k_j}$, with $n = \#( \crl{k\in\mathbb{Z}^d:\abs{k}\leq t}) \leq (2 t + 1)^d$, i.e., $n = O(t^d)$. So for each $\abs{k}\leq t$, we define the feature map $\phi_t: X \to \C^n$ elementwise as $(\phi_t(x))_k = e_k(x)$, where $e_k$ is the $k$-th Fourier component, i.e., $e_k(x) = e^{-2\pi i k^\top x}$.
Consider the class of PSD models
\begin{equation}\label{eq:our-psd-models}
{\cal G}_t = \{ g_{A,t} ~|~ A \in \C^{n \times n}, A \succeq 0\}, \quad g_{A,t}(x) = \phi_t(x)^\ast A \phi_t(x).
\end{equation}
We are thus considering the feature map $\phi_t$ associated with the classical band-limited space of functions. This choice is convenient for our analysis, but there are also many other choices of finite dimensional feature maps for PSD models that can have good approximation properties \citep{rudi2021psd,rudi2020finding}.

We consider continuous, 1-periodic functions $f$ on $\R^d$, i.e., functions satisfying $f(x + k) = f(x)$ for any $x \in \R^d$ and $k \in \Z^d$. We note that these can therefore be identified with continuous, periodic functions on the torus $[0,1]^d$. 
We now introduce $H_\rho$, where $\rho \in \ell_1(\Z^d)$ is a strictly positive summable sequence. The space is a separable Hilbert space of periodic functions defined as $H_{\rho} = \{f  \in L^2(\X) ~|~ \|f\|_\rho < \infty\}$, where $\|f\|^2_\rho = \sum_{k \in \Z^d} |\widehat{f}_k|^2/\rho_k < \infty$ and $\widehat{f}_k = \int_{[0,1]^d} e_k(x) f(x) dx$ is the Fourier series associated to $f$. One classical example of $\rho$ is $\rho_k = (1+\|k\|^2)^{-m}$, with $m > d/2$, corresponding to the Sobolev space $H^m_{2,\textrm{per}}$ of periodic functions whose derivatives up to order $m$ are squared integrable \citep{wahba1990spline}, or the space of periodic entire functions (of order~$1$), corresponding to $\rho_k = \exp(-\sigma \|k\|)$, for some $\sigma > 0$. 

We will soon present a Theorem showing that the PSD models ${\cal G}_t$ can approximate functions of the form $f = \sum_{j=1}^q u_j^2$ for $q \in \N$ and $u_j \in H_\rho$, using a small $t$ depending on the decreasing quantity 
$$R^2_t := \sum_{|k| > t } \rho_k.$$
First, we start with a Lemma concerning the norm $\|\cdot\|_F$ of the pointwise product of functions. This part of the proof is crucial and is handled differently in the other settings \citep[e.g.,][]{rudi2020finding}.
\begin{lemma}\label{lm:pointwise-product-L1}
Let $f, g$ be $1$-periodic functions on $\mc{X}$ with  $\nrm{f}_F, \nrm{g}_F < \infty$ and denote by $f \cdot g$ their pointwise product, i.e.,~$(f \cdot g)(x) = f(x)g(x)$. Then
$$\|f \cdot g\|_F \leq \|f\|_F \|g\|_F.$$
\end{lemma}
\begin{proof}
By the convolution property of Fourier series, $(\widehat{f \cdot g})_k =   \sum_{j \in \Z^d} \widehat{f}_j \widehat{g}_{k-j}$. By the Young inequality for the convolution of discrete sequences, we have  for any two sequences $u,v \in \ell_1(\Z^d)$,   $\sum_{j,k \in \Z^d} |\widehat{u}_j \widehat{v}_{k-j}| \leq (\sum_{k \in \Z^d} |u_k|) (\sum_{k \in \Z^d} |v_k|)$. The result is obtained by applying this inequality on the Fourier series of $f \cdot g$ and noting that  $\sum_{k \in \Z^d} |\widehat{f}_k|$ is exactly $\|f\|_F$, and the same for $g$.
\end{proof}
Now we are ready to state the first theorem that on the approximation error for the PSD models described above. 
\begin{theorem}\label{thm:bound-psd-model}
Let $f(x) = \sum_{j=1}^T u_j(x)^2$ for functions $u_j \in H_\rho$ and $T \in \N$, then
$$ \min_{g \in {\cal G}_t} \|f - g\|_{F} ~\leq~ C'_f \,\, R_t \,.$$
where ${C_f'}^2 = \sum_{j=1}^{T} \|u_j\|_\rho \|u_j\|_F$.
\end{theorem}
\begin{proof}
Denote by $u_{j,t}$ the function, $u_{j,t}(x) = \sum_{|k| \leq t} (\widehat{u}_j)_k e_k(x)$ (a low-pass filtered version of $u$), and by $v_{j,t}$ the $n$-dimensional vector $(v_{j,t})_k = \widehat{u}_{j,t}$, for any $|k| \leq t$.
Now, define $\bar{A} \in \C^{n \times n}$ as
$$\bar{A} = \sum_{j=1}^T v_{j,t} v_{j,t}^\ast.$$
Since, by construction, $v_{j,t}^\ast \phi(x) = \sum_{|k| \leq t} (\widehat{u}_{j})_k e_k(x) = u_{j,t}(x)$, then
$$g_{\bar{A},t}(x) = \phi_t(x)^\ast A \phi_t(x) = \sum_{j=1}^T \phi_t(x)^\ast(v_{j,t} v_{j,t}^\ast) \phi_t(x) = \sum_{j=1}^T (v_{j,t}^\ast \phi_t(x))^2 = \sum_{j=1}^T u_{j,t}(x)^2.$$
Now note that $u_j^2 - u_{j,t}^2 = (u_j + u_{j,t}) \cdot (u_j - u_{j,t})$, then, by using \cref{lm:pointwise-product-L1},
$$\|f - g_{\bar{A},t}\|_{F} = \Big\|\sum_{j=1}^T ( u_j^2 - u_{j,t}^2) \Big\|_F \leq \sum_{j=1}^T (\|u_j\|_F + \|u_{j,t}\|_F)\|u_{j} - u_{j,t}\|_F.$$
We conclude noting that $\|u_{j,t}\|_F \leq \|u_j\|_F$ by construction and, by Cauchy-Schwartz,
$$\|u_j - u_{j,t}\|_{F} = \sum_{|k| > t} |\widehat{u}_{j}| = \sum_{|k| > t} \sqrt{\rho}_t \tfrac{|\widehat{u}_{j}|}{\sqrt{\rho}_t} \leq R_t \|u_{j,t}\|_\rho \leq R_t \|u_{j}\|_\rho.$$
Therefore, 
$\displaystyle \min_{g \in {\cal G}_t} \|f - g\|_F = \min_{A \in \C^{n \times n}, A \succeq 0} \|f - g_{A,t}\|_F \leq \|f - g_{\bar{A},t}\|_F \leq R_t C_f'$.
\end{proof}

The theorem above controls the approximation error of the PSD models of bandwidth $t$ when the target function can be written in terms of a sum of squares of functions belonging to an $H_\rho$ for a given $\rho$. In general it is not clear how to guarantee when a function $f$ can be characterized as a sum of squares of functions in a given space. Luckily, in the case of $m$-times differentiable functions, there exists an easy geometrical characterization. We are going to use this fact, to specify the result above for the case when $f$ is an $m$-times differentiable function. First, we need the following lemma, which is the adaptation to periodic functions of Theorem 2 of \citet{rudi2020finding} (more specifically, of Corollary 2, page 23) . 
\begin{lemma}[\cite{rudi2020finding}]\label{lm:cor2}
Let $f$ be an $m+2$-times differentiable non-negative periodic function. Assume that the minimizers of $f$ in $\X$ are finitely many and with strictly positive Hessian. Then, there exists $Q \in \N$ and $z_1,\dots, z_Q$ periodic $m$-times differentiable functions, such that $f = \sum_{j=1}^Q z_j^2$.
\end{lemma}
The proof of the lemma above is reported in \cref{app:proof-of-lemma-7}. Now we are ready to specify \cref{thm:bound-psd-model} in the case of an $m$-times differentiable function.
\begin{theorem}
\label{thm:approximation-error}
Let $f$ be an $(m + d/2 + 2)$-times differentiable periodic function, with $m > 0$ and let $c_*$ be its global minimum. Assume that the minimizers of $f$ in $\X$ are finitely many and with strictly positive Hessian. Then, for any $t \in \N$,
$$ \min_{g \in {\cal G}_t}\  \|f - c_* - g\|_F ~\leq~ C_f ~ t^{-m},$$
where the constant $C_f$ depends only on $f, m, d$.
\end{theorem}
The proof is self-contained and reported in \cref{app:thm-approx}. It is obtained by first applying Lemma~\ref{lm:cor2} on $f$, and \cref{thm:bound-psd-model} on the resulting characterization. To make this possible and to obtain a sharp rate, a crucial step is to show that the resulting functions belong to the space $H_\rho$ for a specific $\rho$ satisfying $\rho_k \propto |k|^{-2m-d}$, then deriving the bound on the associated residual $R_t$. 

\subsection{The Resulting Problem and the Associated A Priori Guarantees}

Now the problem \cref{eq:problem-relax}, with the PSD models \eqref{eq:our-psd-models} and the $F$ norm, takes the following form
\begin{equation}\label{eq:specific-sdp}
\bar{c} = \max_{c \in \R, \ A \in \C^{n\times{}n}} c - \nrm{f - c - g_A}_F \quad \textrm{ such that }\ \ A\succeq 0,
\end{equation}
and, combining \cref{thm:tight-problem} and \cref{thm:approximation-error} gives the following a priori guarantee
\begin{corollary}
\label{cor:guarantees-appr}
Let $f$ be an $(m + d/2 + 2)$-times differentiable, 1-periodic function with $m > 0$, and let $c_*$ be its global minimum. Also, let $f$ have finitely many minimizers in $\X$, which each have strictly positive Hessian. Then, for any $t \in \N$
$$ 0 ~\leq ~ c_ * -  \bar{c} ~\leq ~C_f \,\, t^{-m}\,.$$
\end{corollary}
Expressing $t$ with respect to $n$, the dimension of the matrix $A$, we have $t = O(n^{1/d})$. The bound above, then reads as 
$$ 0 ~\leq ~ c_* - \bar{c} ~ \leq ~ C' ~ n^{-m/d} \,.$$
This shows that the solution, $\bar{c}$ is always a lower bound of the global minimum $c_*$ and converges to~$c_*$ with a rate depending on the dimension of the matrix $A$ and the degree of differentiability of $f$. E.g., when $m \geq d$, the error goes to zero as quick as $n^{-1}$.
In the following section we see how to solve the optimization problem \cref{eq:specific-sdp} in practice, by making use of the fact that $\|\cdot\|_F$, in the case of the torus, is a sum, which makes it easy to write \cref{eq:specific-sdp} as a stochastic optimization objective. 






\section{Solving the Optimization Problem}\label{sec:optimization}


We now describe the process of solving the optimization problem in  \cref{eq:problem-relax} in the specific case of the $F$ norm and a PSD model $g_A(x) = \phi(x)^\ast A\phi(x)$ parametrized by positive semidefinite $A \in \C^{n\times{}n}$. For now, we consider an arbitrary feature map $\phi$, but we will also contextualize our results in the specific case of $\phi_t$, the map introduced in Section \ref{subsec:psd-models-for-periodic-functions}.
A serious challenge to solving \eqref{eq:specific-sdp} is that computing $\nrm{f - c - g_A}_F$ or its subgradients exactly will typically be intractable because the $F$ norm is the series:
\[
\nrm{f - c - g_A}_F 
= \sum_{k \in \mathbb{Z}^d} \abs{\hat{f}_k - c\indicator{k=0} - \widehat{g_A}_k}.
\]
To circumvent this issue, we recast the problem as a stochastic optimization objective. In particular, we introduce a probability measure, $\pi$, supported on $\mathbb{Z}^d$ and rewrite
\[
\nrm{f - c - g_A}_F 
= \sum_{k \in \mathbb{Z}^d} \pi_k \frac{\abs{\hat{f}_k - c\indicator{k=0} - \widehat{g_A}_k}}{\pi_k} = \E_{k \sim \pi}\brk*{\frac{\abs{\hat{f}_k - c\indicator{k=0} - \widehat{g_A}_k}}{\pi_k}} .
\]
Written this way, we can now attack our objective using any number of methods from the stochastic optimization arsenal, such as projected stochastic gradient ascent.

To see how $\pi$ should be chosen, we first note that $g_A(x) = \inner{A}{\phi(x)\phi(x)^\ast}$, and use $M\upk = \widehat{\phi\phi^\ast}_k \in \mathbb{C}^{n \times n}$ to denote the $k$-th Fourier component of $\phi\phi^\ast$, so that $\widehat{g_A}_k = \inner{A}{M\upk}$. Thus, our optimization problem now reads
\[
\bar{c} = \max_{c\in\R,\ A\in\C^{n\times{}n}} c - \E_{k \sim \pi}\brk*{\frac{\abs{\hat{f}_k - c\indicator{k=0} - \inner{A}{M\upk}}}{\pi_k}} \quad \textrm{ such that }\ A \succeq 0.
\]
Noting that $c$ only appears in two terms, we can also eliminate this variable by solving
\[
\max_c \ c - \abs{\hat{f}_0 - c - \langle A,\,M^{(0)}\rangle} = \hat{f}_0 - \langle A,\,M^{(0)}\rangle.
\]
Putting this all together, we want to solve the stochastic concave maximization problem
\begin{equation}\label{eq:stoch-opt-problem}
\bar{c} = \max_{A\succeq 0}\ 
\E_{k\sim\pi}\brk*{L_k(A)}
\end{equation}
where
\begin{equation}
L_k(A) = \begin{cases}
\frac{1}{\pi_0}\prn*{\hat{f}_0 - \inner{A}{M^{(0)}}} & k = 0 \\
\frac{-1}{\pi_k}\abs*{\hat{f}_k - \inner{A}{M\upk}} & k \neq 0.
\end{cases}
\end{equation}

\begin{algorithm}
\caption{\textsc{Projected Stochastic Gradient Ascent}}
\label{alg:projected-sgd}
\begin{algorithmic}
\STATE Initialize $A_0 = 0$
\FOR{$t=0,1,\dots,T-1$}
\STATE $\tilde{A}_{t+1} = A_t + \eta\nabla L_{k_t}(A_t) \quad\textrm{for}\ \  k_t \sim \pi$
\STATE $A_{t+1} = \min_{A}\nrm{A - \tilde{A}_{t+1}}_{Frob.}$ s.t.~$A \succeq 0$, $\nrm{A}_{Frob.} \leq R$.
\ENDFOR
\STATE \textbf{Return:} $\bar{A}_T = \frac{1}{T}\sum_{t=1}^T A_t$
\end{algorithmic}
\end{algorithm}

Using projected stochastic gradient ascent yields the following error guarantee:
\begin{theorem}\label{thm:optimization-error}
Let $R \geq \nrm{A^*}_{\rm Frob.}$ be an  upper bound the norm of a maximizing $A^*$, and let $\bar{A}_T$ be the output of Algorithm \ref{alg:projected-sgd}, with an optimally-chosen constant stepsize $\eta$ and $\pi_k \propto \nrm{M\upk}_{Frob.} + (1+\sum_{j=1}^d (2 \pi k_j)^{d+1})^{-1}$. Then $\nrm{\bar{A}_T}_{\rm Frob.}\leq R$ and for any $\delta \in (0,1)$, with probability $1-\delta$
\[
\E_{k\sim\pi}\brk*{L_k(\bar{A}_T)} \geq \bar{c} - \frac{20R\log(2/\delta)\prn*{1 + \sum_{k\in\mathbb{Z}^d}\nrm{M\upk}_{Frob.}}}{\sqrt{T}}.
\]
\end{theorem}
The proof, which we defer to Appendix \ref{app:proof-of-thm-optimization-error}, simply requires proving that the functions $L_k$ are Lipschitz-continuous and then appealing to Proposition 2.2 from \citet{nemirovski2009robust}. This result bounds, a priori, the optimization error incurred in trying to estimate $A^*$ which realizes the maximum of \eqref{eq:specific-sdp}. In Section \ref{sec:combining-guarantees}, we combine this with Theorem \ref{thm:approximation-error} and the yet to be presented Theorem \ref{thm:a-posteriori-accuracy} to state our a priori guarantees.


\section{A Posteriori Certification}

Obviously, it is nice to know a priori that our estimate $\bar{A}_T$ will be close to attaining the optimum of \eqref{eq:specific-sdp}. However, with this estimate in hand, what we really want is to compute a lower bound on~$c_*$, so we need to actually evaluate $\E_{k\sim\pi}[L_k(\bar{A}_T)]$, which is non-trivial since $\pi$ has infinite support. 

Things are easier when $f$ and $\phi\phi^\ast$ are \emph{band-limited}, meaning that for some $K$, $\abs{k}>K$ implies $\hat{f}_k = 0$ and $M\upk = 0$. Specifically, we can choose $\pi_k \propto \nrm{M\upk}_{\rm Frob.}$, which is only supported on $\crl{k:\abs{k}\leq K}$, and then easily compute $\E_{k\sim\pi}[L_k(\bar{A}_T)]$ to obtain an exact lower bound on $c_*$.

However, if one or both of $f$ and $\phi\phi^\ast$ are not band-limited, then we are forced to estimate the value of an infinite sum. One approach is to draw samples $k \sim \pi$ and estimate the value using a sample average, and under suitable conditions on $f$ and the matrices $M\upk$, this allows us to accurately estimate the value of the lower bound with high-probability. Alternatively, under stronger conditions on $f$ and the matrices $M\upk$, we can compute $L_k$ for a finite set of $k$'s and deterministically bound the contribution of the remaining, uncomputed terms. The following Theorem indicates the accuracy of these methods:
\begin{theorem}\label{thm:a-posteriori-accuracy}
Let $f$ satisfy the conditions of Theorem \ref{thm:approximation-error} with $m > d/2$. Then for any $K$ and $k_1,\dots,k_K \overset{i.i.d.}{\sim} \pi$, for any $\delta \in (0,1)$ and $A\succeq 0$, with probability  $1-\delta$,
\begin{align*}
c_* \geq \bar{c} \geq \hat{c}_{1-\delta} &:= \frac{1}{K}\sum_{i=1}^K L_{k_i}(A) - \textrm{Err}_{1-\delta} \geq \E_{k\sim\pi}[L_k(A)] - 2\textrm{Err}_{1-\delta} \\
\textrm{Err}_{1-\delta} &:= (\sqrt{d+1}\|f\|_{C^{d+1}(\X)} + \nrm{A}_{\rm Frob.})\sqrt{\frac{2\log(2/\delta)}{K}}\Big(1 + \sum_{k\in\mathbb{Z}^d}\nrm{M\upk}_{\rm Frob.}\Big),
\end{align*}
where $\|f\|_{C^{d+1}(\X)} = \max_{1\leq j\leq d} \max_{1\leq q \leq d+1} \|\frac{\partial^q}{\partial x^q_j} f\|_{L^\infty(\X)}$.
In addition, for any $K$, the following holds deterministically:
\begin{align*}
c_* \geq \bar{c}  \geq \hat{c}_{1} &:= \sum_{k:\abs{k}\leq K} \pi_k L_{k}(A) - \textrm{Err}_1 \geq \E_{k\sim\pi}[L_k(A)] - 2\textrm{Err}_1 \\
\textrm{Err}_1 &:= \sum_{k:\abs{k}>K}\brk*{\abs{\hat{f}_k} + \nrm{A}_{\rm Frob.}\nrm{M\upk}_{\rm Frob.}}
\end{align*}
\end{theorem}
The proof, which we defer to Appendix \ref{app:proof-of-thm-a-posteriori-accuracy}, analyzes $\hat{c}_{1-\delta}$ and $\hat{c}_1$ separately. For the former, we first show that $L_k(A)$ is bounded for each $k$, and then apply Hoeffding's inequality. For the latter, we decompose the sum over $k\in\mathbb{Z}^d$ into those $k$'s with $\abs{k}\leq K$, and those $k$'s with $\abs{k} > K$, and then upper bound this second portion of the sum.

The Theorem shows that the sample average has additive error that decays with $1/\sqrt{K}$ with high probability. Furthermore, this lower bound is tractable given the parameter $\zeta$ and enough knowledge of our feature map for us to upper bound $\sum_{k\in\mathbb{Z}^d}\nrm{M\upk}_{\rm Frob.}$. For the deterministic lower bound on $c_*$, we need to have some control over how quickly $\nrm{M\upk}_{\rm Frob.}$ decays with increasing~$\abs{k}$, but if the feature map is chosen so that this decay is (eventually) rapid, then this lower bound can be tight.
In the particular case of $\mc{G}_t$ introduced in Section \ref{subsec:psd-models-for-periodic-functions}, we show in Appendix \ref{app:bound-our-Mk} that we can bound $\sum_{k\in\mathbb{Z}^d} \nrm{M\upk}_{\rm Frob.} \leq n (8t)^{d}$, and for $K \geq 2t$, $\sum_{k:\abs{k}> K} \nrm{M\upk}_{\rm Frob.} = 0$. Therefore, $\hat{c}_{1-\delta}$ can provide a tight approximation of $\bar{c}$ using $K \gg n (8t)^d$ samples,  and $\hat{c}_1$ can once the maximum bandwidth is set $K \geq 2t$, as long as the Fourier coefficients $\hat{f}_k$ decay sufficiently quickly.

With Theorem \ref{thm:a-posteriori-accuracy}, we can use the solution returned by our optimization algorithm to compute a lower bound on $c_*$, one that holds with high probability and one that holds deterministically. However, to actually compute a certificate of the accuracy of our lower bound, we also need an upper bound on $c_*$. Getting \emph{some} upper bound on $c_*$ is as easy as evaluating $f(x)$ at any point $x$, although most $x$'s will not be close to minimizing $f$, so this may not give us much information about~$c_*$. Of course, there are many better ways, and the bulk of the non-convex optimization literature is devoted to designing algorithms for computing approximate minimizers of $f$, i.e., upper bounds on~$c_*$. Upper bounds for $c_*$ are easier to produce, $f(x_0)$ for any point $x_0 \in \X$ is a valid upper bound. We can use the point $x_0$ produced for example by \cite{rudi2020finding}, that converges provably to a global minimizer with a rate that avoids the curse of dimensionality.  In our experiments (in low dimensions), we simply compute $f(x_1),\dots,f(x_N)$ for $N$ random points and upper bound $c_* \leq \min_i f(x_i)$, which allows for tight enough certificates.

\section{A Priori and A Posteriori Guarantees}\label{sec:combining-guarantees}

In the previous sections, we have described a method for estimating $c_*$ in the case of periodic functions on $[0,1]^d$, and all the pieces are in place to state our method's a priori and a posteriori guarantees. To summarize so far:
\begin{enumerate}[topsep=0pt,itemsep=-1ex,partopsep=0ex,parsep=1ex]
\item Theorem \ref{thm:tight-problem} shows that the solution of the relaxed problem \eqref{eq:problem-relax}, $\bar{c}$, is a lower bound on $c_*$ which is tight up to the error of approximating $f-c_*$ with the class of non-negative functions, $\mc{G}$.
\item In Theorem \ref{thm:approximation-error}, we bound this approximation error for smooth, periodic functions with respect to $\mc{G}_t$, the class of PSD models defined in \eqref{eq:our-psd-models} with the band-limited kernel $\phi_t$. 
\item But, we need to actually solve \eqref{eq:problem-relax} defined using the $F$ norm and $\mc{G}_t$. So, in Theorem \ref{thm:optimization-error}, we bound the optimization error of the solution returned by projected stochastic gradient ascent.
\item However, our optimization algorithm returns the parameters $\bar{A}_T$ of a PSD model, and to compute a lower bound on $c_*$ we need to actually evaluate the value of the objective at $\bar{A}_T$. So, finally, Theorem \ref{thm:a-posteriori-accuracy} bounds the estimation error when using $\bar{A}_T$ to estimate lower bounds $\hat{c}_{1-\delta}$ and $\hat{c}_1$ on $c_*$ that holds with high probability and deterministically, respectively.
\end{enumerate}
Therefore, our a priori guarantees amount to combining 
(Approximation Error) $+$ (Optimization Error) $+$ (Estimation Error).
On the other hand, given any PSD model parameters, $A$, we can evaluate an a posteriori bound on the error by upper bounding $c_* \leq f(x)$ for any $x$ and lower bounding $c_*$ using Theorem \ref{thm:a-posteriori-accuracy}. The following Corollary summarizes these guarantees:
\begin{corollary}\label{cor:all-together}
For the $F$ norm and family of PSD models $\mc{G}_t$ defined using $\phi_t$, under the conditions of Theorems \ref{thm:approximation-error}, \ref{thm:optimization-error}, and \ref{thm:a-posteriori-accuracy}, let $\hat{c}_{1-\delta}(\bar{A}_T)$ and $\hat{c}_1(\bar{A}_T)$ be lower bound estimates defined in Theorem \ref{thm:a-posteriori-accuracy}. Then for any $\delta \in (0,1)$, we provide the following a priori guarantee with probability $1-2\delta$:
\[
c_* \geq \hat{c}_{1-\delta}(\bar{A}_T) 
\geq c_*  -  C_f t^{-m} - C_d n t^{d} \prn*{\frac{20R\log(2/\delta)}{\sqrt{T}} + \frac{2(\sqrt{d+1}\|f\|_{C^{d+1}(\X)} + R)\sqrt{2\log(2/\delta)}}{\sqrt{K}}},  
\]
with $C_d = 8^d$. At the same time, given any point $x$ and parameters $A$ for the PSD model, we guarantee a posteriori that $f(x) \geq c_* \geq \hat{c}_{1}(A)$ and $f(x) \geq c_* \geq \hat{c}_{1-\delta}(A)$ with probability $1-\delta$.
\end{corollary}
The Corollary follows immediately by combining Theorems \ref{thm:approximation-error}, \ref{thm:optimization-error}, and \ref{thm:a-posteriori-accuracy}. Since $t = O(n^{1/d})$, by choosing $T = K = O(n^{4 + 2m/d})$, we have
$$ c_* \geq \hat{c}_{1-\delta}(\bar{A}_T) 
\geq c_* - C' n^{-m/d},$$
when $n$ is the dimension of the matrix $\bar{A}_T$. In this case the algorithm has a complexity that is $O(Tn^3 + K n^2) = O(n^{7 + 2m/d})$. In particular, for the class of $(m+d/2+2)$-times differentiable functions with $m > d/2$, we achieve a bound $c_* \geq \hat{c}_{1-\delta}(\bar{A}_T) 
\geq c_* - C' n^{-1}$, with a computational cost of $O(n^{8})$. There is a lot of room for improvement in the constants of the exponents, but the considered algorithm shows that it is possible to obtain the global optimum of a function with both a posteriori guarantees and an a priori error rate that is adaptive to the degree of differentiability of the function to minimize and that avoids the curse of dimensionality for very smooth functions.


\section{Empirical Evaluation}\label{subsec:empirical}

Finally, we apply our method to two simple non-convex optimization problems in one and two dimensions. The results are summarized in Figures \ref{fig:empirical-results1d} and \ref{fig:empirical-results2d}, and all of the details of the experiments are deferred to Appendix \ref{app:experimental-details}, in which we describe a new feature map $\phi$, and describe a more practical algorithm for solving \eqref{eq:specific-sdp} based on reparametrizing $A = UU^\ast$ \citep{burer2003nonlinear}.

\section{Discussion}

\paragraph{Convex duality.} Following~\citet{rudi2020finding}, we can provide a dual interpretation to the use of PSD models. Indeed, the minimimization problem we solve is
  $$
  \inf_{ \mu\  \rm{ probability } \ {\rm  measure} } \int_{\mc{X}}  f(x)  d\mu(x)
  $$
  which we reformulate as
   $$
  \inf_{ \mu \ {\rm signed}\  {\rm measure} } \int_{\mc{X}}  f(x)  d\mu(x)
  \mbox{ such that }    \int_{\mc{X}} d\mu(x) = 1 \mbox{ and }    \int_{\mc{X}} \Phi(x) \Phi(x)^*   d\mu(x) \succcurlyeq 0.
  $$
  Given that we expect the solution of the original problem to Dirac measures supported at global minimizers, we can add constraint that are satisfied by Diracs, such as, 
  $$
  \int_{\mc{X}}|d\mu(x)| \leqslant 1 \mbox{ or }   \Omega(\mu)  \leqslant 1,
  $$
  for any norm $\Omega$ on signed measure that is larger than the total variation norm.
  The first constraint leads to a dual problem
 $$
  \sup_{c \in \rb, \ B \succcurlyeq 0}  c - \big\| f - c 1 - \phi(\cdot)^\top B \phi(\cdot) \big\|_\infty,
  $$
  while the second one leads to
    $$
  \sup_{c \in \rb, \ B \succcurlyeq 0}  c - \Omega^\ast\big( f - c 1 - \phi(\cdot)^\top B \phi(\cdot) \big).
  $$
  It turns out that the dual of the $S$ norm and of $F$ norm are domininating the total variation norm, and thus have a dual interpretation. Thus, our method for obtaining a posteriori certificates directly extends to optimization problems that are defined through probabilty measures and already tackled by kernel sum-of-squares, such as optimal transport \citep{vacher2021dimension}, or optimal control \citep{berthier2021infinite}.

  \paragraph{Comparison to previous work on kernel sums-of-square.}
Compared to \citet{rudi2020finding}, the subsampling is now done differently than before: the constraint $\int_{\mc{X}}  \phi(x) \phi(x)^* d\mu(x) \succcurlyeq 0$ is replaced by the projection on the span of $\phi(x^{(1)},\dots,\phi(x^{(n)})$ being a positive semidefinite matrix. This is a relaxation in the dual, which still leads to a lower bound for the optimization problem.
  
  This also suggests a candidate optimal solution when applied to the torus. Indeed, at optimality, we expect $\mu$ to be close to a Dirac at $x_0$, and then (in 1D for simplicity), $\hat{\mu}_1$ should be close to $e^{-2i\pi x_0}$, and we can read off a candidate minimizer as the argument of the first Fourier coefficients (we could imagine using more than one).



\begin{figure}
     \centering
     \begin{minipage}{0.49\textwidth}
         \centering
         \includegraphics[width=\textwidth]{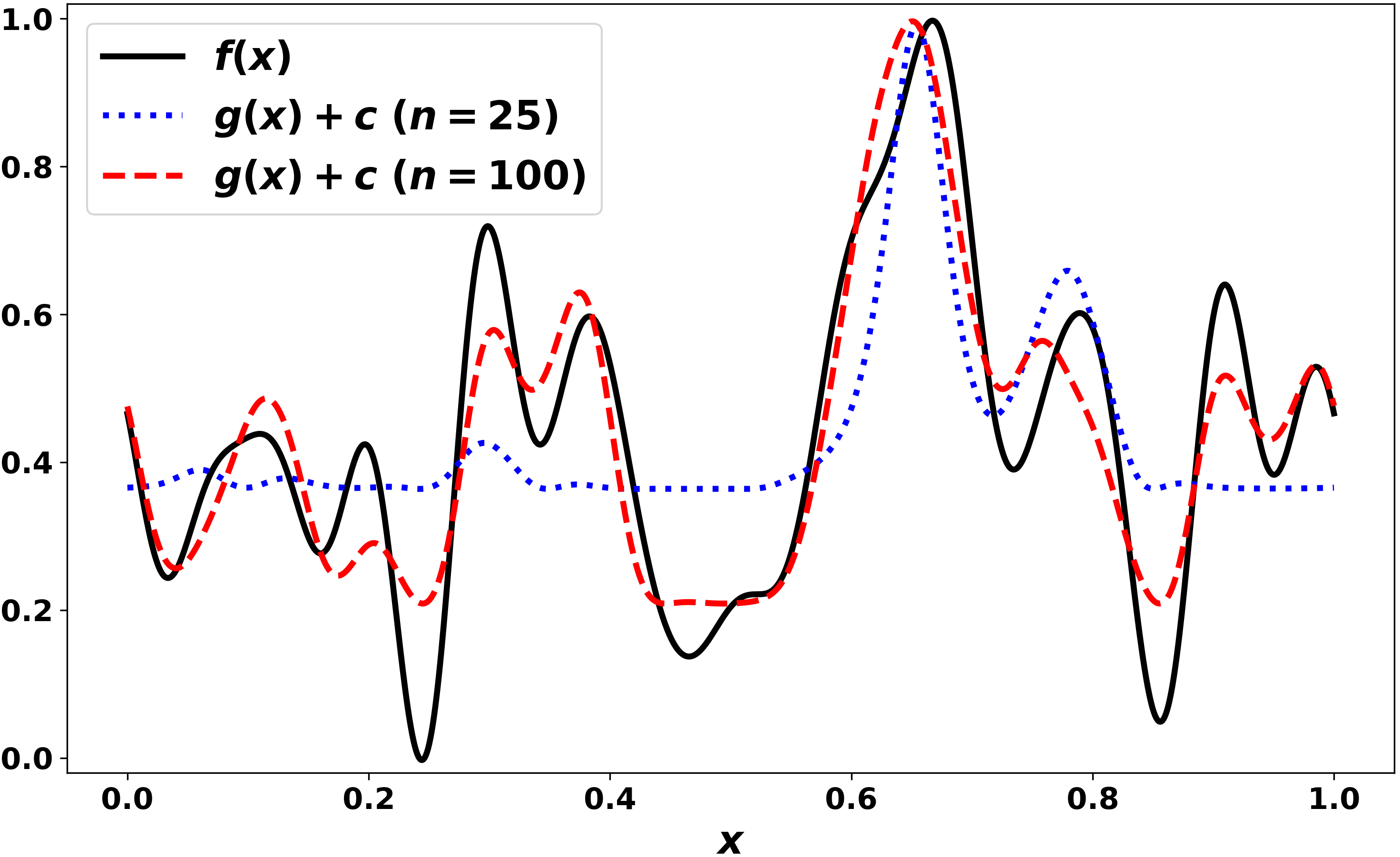}
     \end{minipage}
     \hfill
     \begin{minipage}{0.49\textwidth}
         \centering
         \includegraphics[width=\textwidth]{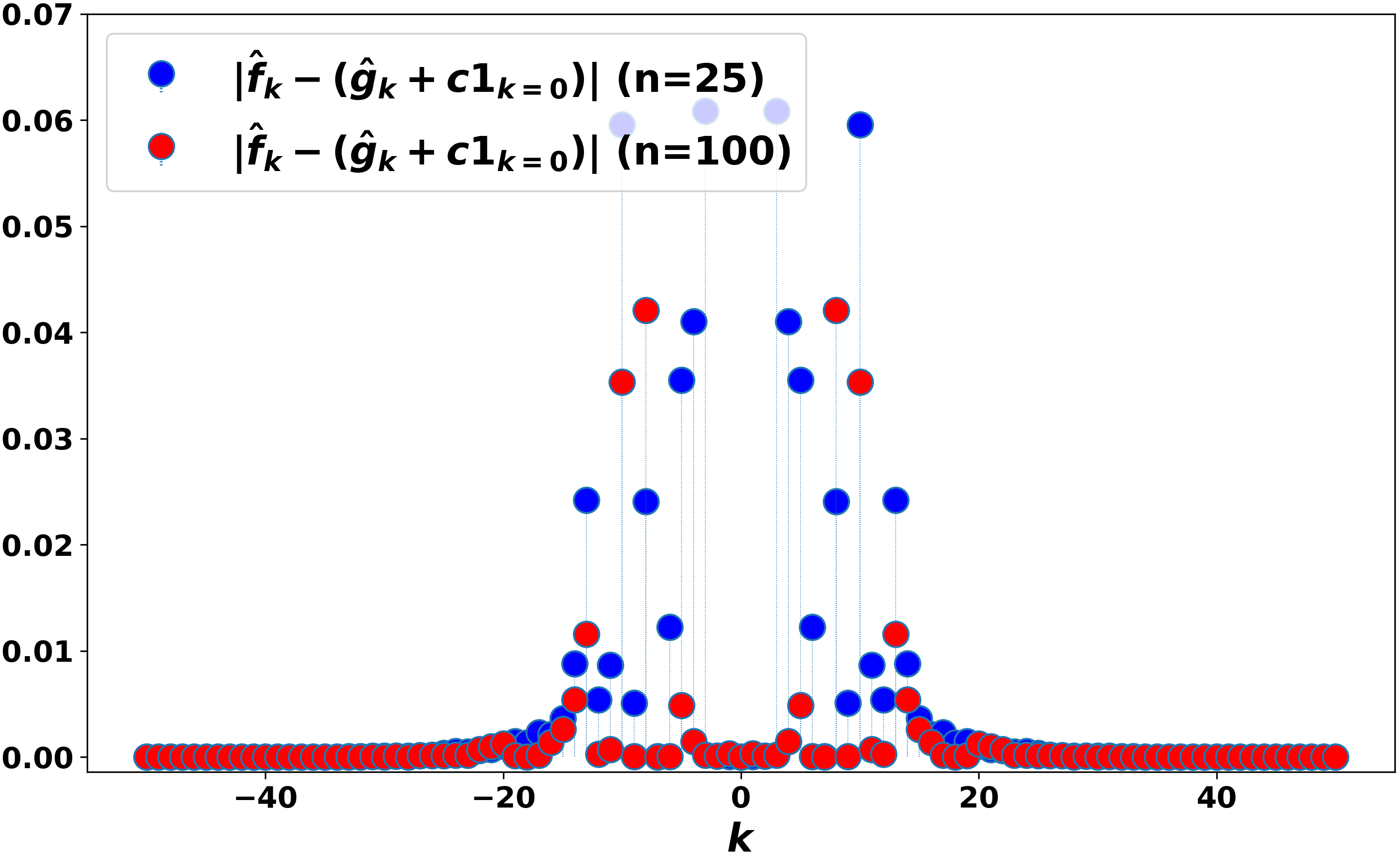}
     \end{minipage}
        {\vspace{-2.5mm}\small\caption{For $f:\R\to\R$ as described in Appendix \ref{app:experimental-details}, we use PSD models using feature maps of dimension $n=25$ and $n=100$. To the left, we see that for $n=25$, the model $g(x)+c$ does not approximate $f$ well, so our a posteriori error guarantee is $0.24$. But, when $n=100$, the model approximates $f$ much better, and our a posteriori guarantee is $0.01$. To the right, we plot the absolute difference between $\hat{f}_k$ and $\widehat{(g+c)}_k$ for small $k$'s, which is what drives the difference in performance between $n=25$ and $n=100$.
        \label{fig:empirical-results1d}}\vspace{-1mm}}
\end{figure}

\begin{figure}
\centering
\includegraphics[width=0.6\textwidth]{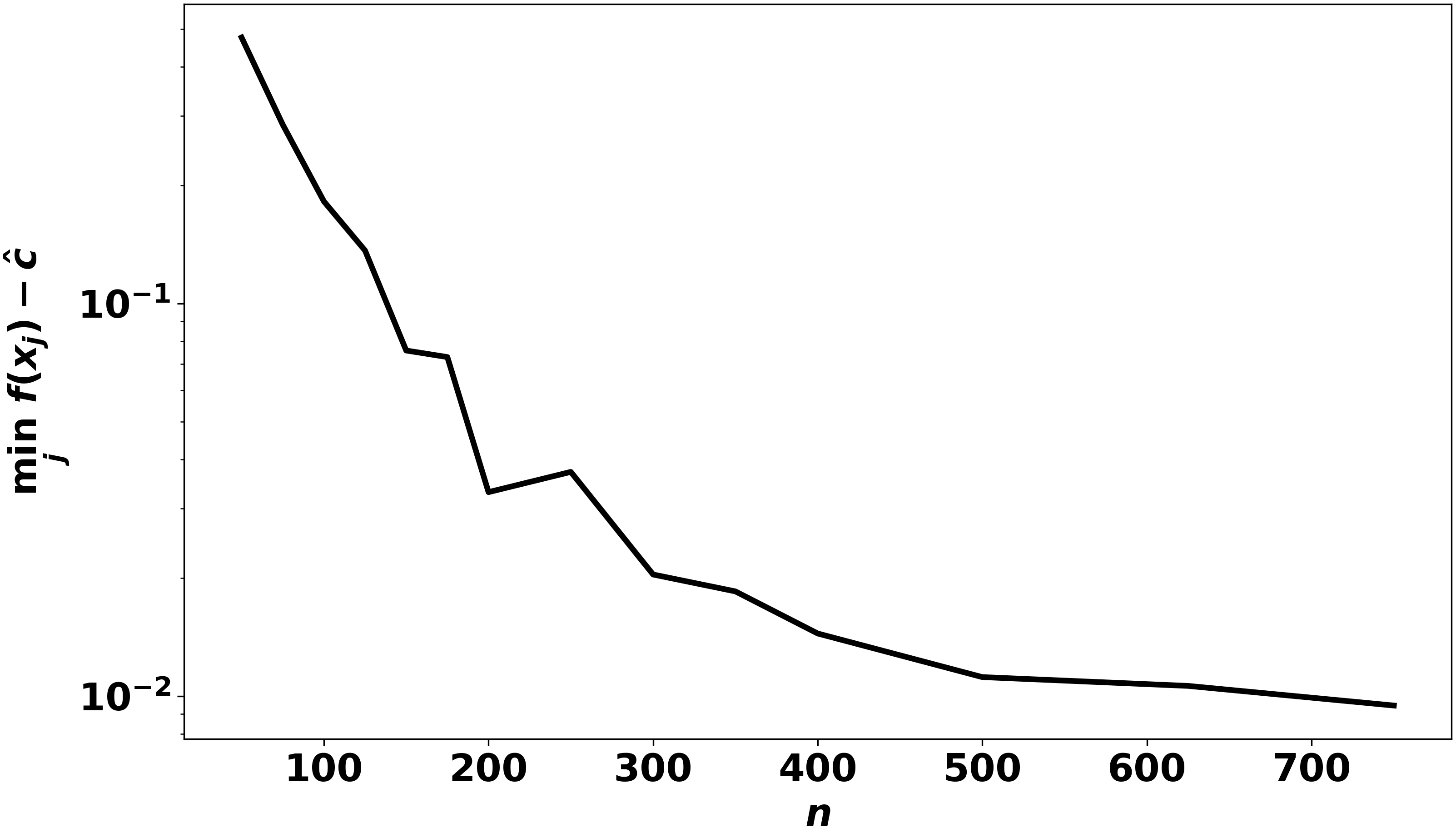}
{\vspace{-3mm}\small\caption{For $f:\R^2\to\R$ as described in Appendix \ref{app:experimental-details}, we plot the a posteriori error guarantee of our algorithm's estimate vs.~$n$, i.e~$\min_j f(x_j) - \hat{c}$ which is the difference between the minimum value of $f$ achieved on a random grid of points, which upper bounds $c_*$, and our estimate of the minimum, $\hat{c}_1$, which lower bounds $c_*$. The error is at most 10\% of the function's range with $n=150$, and can be made less than 1\% with $n=750$.\label{fig:empirical-results2d}}\vspace{-2mm}}
\end{figure}

\paragraph{Acknowledgements}
This work was supported by the French government under the management of the Agence Nationale de la Recherche as part of the “Investissements d’avenir” program, reference ANR-19-P3IA-0001 (PRAIRIE 3IA Institute). We also acknowledge support from the European Research Council (grants SEQUOIA 724063 and REAL 947908).



\bibliography{bibliography}

\newpage

\appendix

\section{Proof of Lemma~\ref{lm:cor2}}\label{app:proof-of-lemma-7}

\begin{proof}
The proof of the adaptation of Corollary 2 of \cite{rudi2020finding} to the periodic setting is organized in four steps that are summarized in this paragraph. By translation $\tilde{f}(x) = f(x-b)$ for a suitable vector $b \in \mc{X}$, we show that all the global minima are contained in a closed set $\mc{A}$ strictly contained in an open set $\Omega$ strictly contained in the closed set $[0,1]^d$. Then Corollary 2 of \cite{rudi2020finding} shows that there exist $q$ functions $u_1,\dots u_q \in C^{m+2}(\R^d)$ such that $\tilde{f}(x) = \sum_{j=1}^q u_j(x)^2$ for any $x \in \Omega$. Then we build two bump functions such that $0 \leq \eta, \nu \leq 1$ and $\eta^2 + \nu^2 = 1$ on $\R^d$, moreover $\eta = 0$ on $\mc{X}\setminus \Omega$ and $1$ on $\mc{A}$ (so $\nu = 1$ on $\mc{X}\setminus \Omega$ and $0$ on $\mc{A}$). Then we create a new function $\tilde{v} = \sqrt{\tilde{f}} \cdot \nu$ and we show that its periodic extension $\tilde{v}_{\textrm{per}}$ satisfies $\tilde{v}_{\textrm{per}} \in C^{m+2}(\R^d)$. We prove the same for $\tilde{u}_{j,\textrm{per}}$, the periodic extension of $\tilde{u}_{j} = u_j \cdot \nu$. The result is obtained by noting that $\tilde{u}_{1,\textrm{per}}, \dots, \tilde{u}_{q,\textrm{per}}, \tilde{v}_{\textrm{per}}$ are $m$-times differentiable periodic functions that satisfy $f(x) = \tilde{v}_{\textrm{per}}(x+b)^2 + \sum_{j} \tilde{u}_{j,\textrm{per}}(x+b)^2$ for any $x \in \R^d$.

\paragraph{Step 1. Translating the functions, applying Corollary 2 of \cite{rudi2020finding}.}
Let $s \in \N$ and $M = \{x_1,\dots,x_s\}$ be the set of minimizers on $[0,1)^d$. First, we work with a translated periodic function $\tilde{f} = f(x-b)$, where $b = (\tau/2, \dots, \tau/2) \in \R^d$
and $\tau$ is the minimum distance of a point in $M$ from $[0,1]^d \setminus [0,1)^d$ (note that $\tau > 0$ by construction). Let $\mc{A} = [2\tau/3,1-2\tau/3]^d$ and $\Omega = \cup_{x \in (\tau/2,1-\tau/2)^d} B_{\tau/6}(x)$, where $B_{\tau/6}(x)$ is the open ball of radius $\tau/6$ centered in $x$. Note that $\mc{A} \subset \Omega \subset \mc{X}$ and that $\mc{A}, \mc{X}$ are closed, while $\Omega$ is open.
Now in the translated version $\tilde{f}$ all the zeros are in $\mc{A}$. By applying Corollary 2 of \cite{rudi2020finding} on $\tilde{f}$ and $\Omega$, we obtain that there exists $q \in \N$ and  $u_1,\dots, u_q \in C^{m}(\R^d)$ such that
$$\tilde{f}(x) = \sum_{j=1}^q u_j(x)^2, \quad \forall x \in \Omega.$$

\paragraph{Step 2. Building the bump functions.}
Let now $\alpha, \beta$ be two infinitely differentiable non-negative functions on $\R^d$ such that
$\alpha = 0$ on $\R^d \setminus \Omega$ and is strictly positive on $\Omega$, while $\beta = 0$ on $\mc{A}$, strictly positive on $\R^d \setminus \mc{A}$. Since $\alpha^2 + \beta^2 > 0$ on $\R^d$, the function $\sqrt{\cdot} \in C^\infty((0,\infty))$ and $\alpha, \beta \in C^\infty(\R^d)$, then $\eta = \alpha/\sqrt{\alpha^2 + \beta^2}$ and $\nu = \beta/\sqrt{\alpha^2 + \beta^2}$ are $C^\infty(\R^d)$ and satisfy $\eta = 0$ on $\R^d \setminus \Omega$, $\eta = 1$ on $\mc{A}$, analogously $\nu = 1$ on $\R^d \setminus \Omega$ and $0$ on $\mc{A}$, moreover $\eta^2 + \nu^2 = 1$ on $\R^d$.

\paragraph{Step 3. Construction of $\tilde{v}_{\textrm{per}}$ and $\tilde{u}_{j,\textrm{per}}$.}
Let $C = [\tau,1-\tau]^d$. By construction, $\{x_1,\dots,x_s\} \subset C$ and so $\tilde{f} > 0$ on the set $\mc{X} \setminus C$. Then $(\tilde{f})^{1/2} \in C^{m+2}(\mc{X}\setminus C)$ since it is the composition of $\sqrt{\cdot} \in C^\infty((0,\infty))$ and $\tilde{f}$ that is $\tilde{f} > 0$ on $\mc{X}\setminus C$.
Then $\tilde{v} = (\tilde{f})^{1/2} \cdot \nu \in C^{m+2}(\mc{X})$, since $(\tilde{f})^{1/2}$ is $m+2$ times differentiable on the set $\mc{X} \setminus C$ and $\nu$ is infinitely differentiable and $0$ on $\mc{A} \supset\supset C$.
Denote by $\tilde{v}_{\textrm{per}}$ the periodic extension of $\tilde{v}$, i.e. $\tilde{v}_{\textrm{per}}(x+k) = \tilde{v}(x)$ for any $x \in \mc{X}$ and $k \in \Z^d$. 

We now prove that $\tilde{v}_{\textrm{per}}$ satisfies $\tilde{v}_{\textrm{per}} \in C^{m+2}(\R^d)$. First note that  it is $m+2$ times differentiable in the interior of each cube $\mc{X} + k$ for $k \in Z^d$, since $\tilde{v}$ has this property on the interior of $\mc{X}$. Moreover it has the same property also in a neighbourhood of the set $S + k$ with $S = [0,1]^d \setminus (0,1]^d$ and $k \in \Z^d$. Indeed, let $B_{\tau/6}(x+k)$ be the open ball of radius $\tau/6$ around $x+k$, with $x \in S$ and $k \in \Z^d$. On $B_{\tau/6}(x+k)$ the function $\tilde{v}$ is equal to $(\tilde{f})^{1/2}$, which in that region is $m+2$ times differentiable, since we are in a translation of $\mc{X} \setminus \Omega$.

Define now $\tilde{u}_j = u_j \cdot \eta$ and denote by $\tilde{u}_{j,\textrm{per}}$ its periodic extension. Similary to the case of $\tilde{v}_{\textrm{per}}$, since $\tilde{u}_j$ is identically $0$ on $\mc{X} \setminus \Omega$ and it is $m$-times differentiable on $\mc{X}$, we can prove that the periodic extension of $\tilde{u}_j$ satisfies $\tilde{u}_{j,\textrm{per}} \in C^{m}(\R^d)$.

\paragraph{Step 4. Conclusion.}
Note that $\R^d = \cup_{k \in \Z^d} \big\{ \mc{X} + k \big\}$. Then, by expanding the definitions, we have that for all $x \in \mc{X}, k \in \Z^d$,
\begin{align*}
\tilde{v}_{\textrm{per}}(x + k)^2 & + \sum_{j=1}^q \tilde{u}_{j,\textrm{per}}(x+k)^2  =  \tilde{v}_{j}(x)^2 + \sum_{j=1}^q \tilde{u}_{j}(x)^2 \\
&=  ((\tilde{f})^{1/2}(x)\nu(x))^2  + \sum_{j=1}^q (u_j(x) \eta(x))^2 =  \tilde{f}(x)\nu(x)^2  + \eta(x)^2\sum_{j=1}^q u_j(x)^2 \\
& = \begin{cases}
\tilde{f}(x)  & x \in \mc{X} \setminus \Omega\\
\tilde{f}(x)\nu(x)^2  + \eta(x)^2\sum_{j=1}^q u_j(x)^2 & x \in \Omega \setminus \mc{A}\\
\sum_{j=1}^q u_j(x)^2 & x \in  \mc{A}
\end{cases}\\
& = \tilde{f}(x),
\end{align*}
where in the last two steps we use the fact that $\sum_j u_j(x)^2 = \tilde{f}(x)$ on $x \in \Omega$ and the fact that $\eta^2 + \nu^2 = 1$ everywhere and, in particular, $\eta = 0$ on $\mc{X} \setminus \Omega$, and on $\mc{A}$, while $\nu = 1$ on $\mc{X} \setminus \Omega$ and $0$ on $\mc{A}$.
The proof is concluded by taking the translated version of by the vector $b$. I.e. $z_j(x) = \tilde{u}_{j,\textrm{per}}(x+b)$ for any $x \in \R^d$ and $j = 1,\dots, q$ and moreover $z_{q+1}(x) = \tilde{v}_{\textrm{per}}(x+b)$ for any $x \in \R^d$, and $Q = q+1$.
\end{proof}

\section{Proof of \cref{thm:approximation-error}}\label{app:thm-approx}

\begin{proof}
First, let $s = m+d/2$. Note that, by applying Lemma \ref{lm:cor2} to $f$, we have that there exist $Q \in \N$ functions  $u_1, \dots, u_Q$, that are periodic, $s$-times differentiable and which provide a new characterization of $f$ as $f = \sum_{j=1}^Q u_j^2$. The desired result is obtained by applying \cref{thm:bound-psd-model} to this characterization of $f$. To apply \cref{thm:bound-psd-model}, we need to find a suitable $\rho \in \ell_1(\Z^d)$ such that $H_\rho$ contains $u_1,\dots, u_Q$. In particular, we choose $\rho_k = (1 + \sum_{j=1}^d (2\pi k_j)^s)^{-2}$.
We prove now that $u_j \in H_\rho$ and we characterize the resulting convergence rate.

Denote $i = \sqrt{-1}$ and by $\partial^q_j u$ the function $\partial^q_j u = \frac{\partial^q}{\partial x_j^q} u$, for all $j \in \{1,\dots,d\}$ and $q \in \{0,\dots, s\}$, and $u$ an $s$-times differentiable periodic function.
Denote by $\|u\|_{C^s(\mc{X})}$ the following norm $\|u\|_{C^s(\mc{X})} = \max_{1 \leq j \leq d} \max_{0\leq q \leq s} \|\partial^q_j u\|_{L^\infty(\mc{X})}$.
With the notation we are using of the Fourier series we have $(\widehat{ \partial^q_j  u})_k = (2 \pi i k_j)^q \widehat{u}_k$ \citep[see, e.g., Prop.~3.1.2 of][]{grafakos2008classical}. Now, by the Plancherel's identity \citep[Prop 3.1.16 of][]{grafakos2008classical} and the fact that $\mc{X}$ has volume equal to $1$,
$$
\sum_{k \in \Z^d} (2 \pi k_j)^{2q} |\widehat{u}_k|^2 = \sum_{k \in \Z^d} |(2 i \pi k_j)^q \widehat{u}_k|^2  = \int_{\mc{X}} |\partial^q_j u|^2 dx \leq \|\partial^q_j u\|^2_{L^\infty(\mc{X})} \leq \|u\|^2_{C^s(\mc{X})}  < \infty,
$$
where the norm $\|u\|_{C^s(\mc{X})}$ is finite, since, for any $u$ that is periodic and $s$-times differentiable, we have that $\frac{\partial^q}{\partial x_j^q} u$, with $q \leq s$, is also continuous and periodic, so uniformly bounded on $\mc{X}$. Now, since $(\sum_{j=0}^d a_j)^2 \leq c_d \sum_{j=0}^d a_j^2$ for any $a_j \geq 0$, with $c = d+1$, by expanding the definition of $\rho_k$, 
\begin{equation}\label{eq:bound-L2-Fourier}
\sum_{k \in \Z^d} \frac{|\widehat{u}_k|^2}{\rho_k} = \sum_{k \in \Z^d} \Big|\widehat{u}_k + \sum_{j=1}^d (2\pi k_j)^s\widehat{u}_k\Big|^2 \leq c\Big(\sum_{k \in \Z^d} |\widehat{u}_k|^2 + \sum_{j=1}^d \sum_{k \in \Z^d} |(2 i \pi k_j)^s\widehat{u}|^2\Big) \leq c\|u\|^2_{C^s(\mc{X})}.
\end{equation}
This proves that the functions $u_1,\dots, u_Q$, that are periodic and $s$-times differentiable, belong to the space $H_\rho$. Now we can apply \cref{thm:bound-psd-model}, which gives $\min_{g \in {\cal G}_t} \|f - g\|_F \leq C'_f R_t$, for any $t \in \N$, where now $C'^2_f = c\sum_{j=1}^Q \|u_j\|_F\|u_j\|_{C^m(\mc{X})}$ and $R_t$ is bound as follows. 

Since, $(\sum_{j=0}^d a_j^s)^2 \geq c_s (\sum_{j=0}^d a_j)^{2s}$ for any $a_j \geq 0$, with $c_s = (d+1)^{-2(s-1)}$, \citep[see, e.g., page 11 of][]{grafakos2008classical}, and the cardinality of the set of vectors in $\Z^d$ summing up to a given number corresponds to $\#\big\{k~\big|~|k| = r\big\} = \binom{r+d-1}{d-1} \leq C_d r^{d-1}$ for any $r,d \in \N$ \cite[e.g., page 52 of][]{brualdi2009introductory}, with $C_d = (2e)^{k-1}$, we have
\begin{align*}
R^2_t = \sum_{|k| > t} \rho_k & \leq \sum_{|k| > t} \frac{1}{1 + c_s |k|^{2s}} = \sum_{r>t} \frac{\#\big\{k~\big|~|k| = r\big\}}{1 + c_s r^{2s}} \leq \sum_{r > t}  \frac{C_d r^{d-1}}{1 + c_s r^{2s}} \leq \frac{C_d}{(s-d)\,c_s} t^{-(2s-d)},
\end{align*}
where, in the last step, we used the fact that $\frac{C_d r^{d-1}}{1 + c_s r^{2s}} \leq \frac{C_d}{c_s} r^{-(2s-d+1)}$, moreover, $\sum_{r > t} r^{-(2s - d+1)} \leq \int_t^\infty x^{-(2s-d+1)} dx = \frac{t^{-(2s-d)}}{2s-d}$. The final constant is then $C^2_f = C'^2_f C_d/((s-d) c_s)$.
\end{proof}

\section{Proof of Theorem \ref{thm:optimization-error}}\label{app:proof-of-thm-optimization-error}
\begin{proof}
When $L_k$ is concave and $G$-Lipschitz w.r.t.~the Frobenius norm for all $k$, then the average of the iterates of projected stochastic gradient ascent with optimal constant stepsize applied to a problem of the form in \cref{eq:stoch-opt-problem} will have error bounded by \citep[][Proposition 2.2]{nemirovski2009robust}
\begin{equation}\label{eq:general-sgd-bound-convergence-proof}
\bar{c} - \E_{k\sim\pi}[L_k(\bar{A}_T)] \leq a_0\cdot\frac{GR\log(2/\delta)}{\sqrt{T}}
\end{equation}
with probability at least $1-\delta$.
For our particular objective, it is easy to see that for all $k \in \mathbb{Z}^d$, 
\[
\sup_{A,A'}\frac{\abs{L_k(A) - L_k(A')}}{\nrm{A - A'}_{Frob.}} \leq \frac{\nrm{M\upk}_{Frob.}}{\pi_k}.
\]
Therefore, with our choice of $\pi_k \propto \nrm{M\upk}_{Frob.} + (1+\sum_{j=1}^d (2 \pi k_j)^{d+1})^{-1}$, we can bound the parameter of Lipschitz continuity for all $k$ by 
\[
G 
\leq \sum_{k\in\mathbb{Z}^d} \brk*{\nrm{M\upk}_{Frob.} + \prn*{1+\sum_{j=1}^d (2 \pi k_j)^{d+1}}^{-1}}
\leq 1 + \sum_{k\in\mathbb{Z}^d} \nrm{M\upk}_{Frob.}
\]
Plugging this into \cref{eq:general-sgd-bound-convergence-proof} completes the proof.
\end{proof}

\section{Proof of Theorem \ref{thm:a-posteriori-accuracy}}\label{app:proof-of-thm-a-posteriori-accuracy}

\begin{proof}
First, we prove the high-probability bound. We begin by arguing that $L_k(A)$ is bounded. Let $\mu_k = (1+\sum_{j=1}^d (2 \pi k_j)^{d+1})^{-1}$  so that $\pi_k \propto \nrm{M\upk}_{Frob.} + \mu_k$. For each $k \neq 0$, we have
\begin{align*}
\abs{L_k(A)} 
&= \frac{1}{\pi_k}\abs*{\hat{f}_k - \inner{A}{M\upk}} \\
&\leq \frac{\sum_{k'\in\mathbb{Z}^d}\brk*{\nrm{M^{(k')}}_{Frob.} + \mu_{k'}}}{\nrm{M\upk}_{Frob.} + \mu_k}\prn*{\abs{\hat{f}_k} + \nrm{A}_{Frob.}\nrm{M\upk}_{Frob.}} \\
&\leq \prn*{1 + \sum_{k'\in\mathbb{Z}^d}\nrm{M^{(k')}}_{Frob.}}\prn*{\frac{\abs{\hat{f}_k}}{\mu_k} + \nrm{A}_{Frob.}}.
\end{align*}
Now we need to bound $\abs{\hat{f}_k}/\mu_k$. Note that, by applying \cref{eq:bound-L2-Fourier}, with $u = f$, $s = d+1$ and $\rho_k = \mu^2_k$, we have that for any $k \in \Z^d$,
$$
\frac{|f_k|}{\mu_k} \leq \left(\sum_{k \in \Z^d} \frac{|f_k|^2}{\mu^2_k}\right)^{1/2} \leq \sqrt{d+1}\|f\|_{C^{d+1}(\X)},
$$
where $\|f\|_{C^{d+1}(\X)} = \max_{j=1,\dots,d} \max_{q=1,\dots,d+1} \|\frac{\partial^q}{\partial x^q_j} f\|_{L^\infty(\X)}$.
The result then follows by Hoeffding's inequality: for any $\delta \in (0,1)$
\begin{multline*}
\P\bigg(\abs*{\E_{k\sim\pi}[L_k(A)] - \frac{1}{K}\sum_{i=1}^K L_{k_i}(A)} \geq\\ (\sqrt{d+1}\|f\|_{C^{d+1}(\X)} + \nrm{A}_{Frob.})\sqrt{\frac{2\log(2/\delta)}{K}}\Big(1 + \sum_{k\in\mathbb{Z}^d}\nrm{M\upk}_{Frob.}\Big)\bigg) \leq \delta.
\end{multline*}
Rearranging and noting that  $\E_{k\sim\pi}[L_k(A)] \leq \bar{c} \leq c_*$, completes the first half of the proof.

For the second set of bounds, we note that
\begin{align*}
\bar{c}
&\geq \E_{k\sim\pi}[L_k(A)] \\
&= \sum_{k:\abs{k}\leq K} \pi_k L_k(A) - \sum_{k:\abs{k} > K} \abs*{\hat{f}_k - \inner{A}{M\upk}} \\
&\geq \sum_{k:\abs{k}\leq K} \pi_k L_k(A) - \sum_{k:\abs{k} > K} \prn*{\abs{\hat{f}_k} +\nrm{A}_{Frob.}\nrm{M\upk}_{Frob.}} \\
\end{align*}
This completes the proof.
\end{proof}

\section{Bound on $\nrm{M\upk}_{Frob.}$ for $\phi_t$}\label{app:bound-our-Mk}
The $k_1,k_2$ entry of $(\phi\phi^\ast)(x)$ is equal to $e_{k_1}(x)e_{k_2}(x)^\ast$, so 
\begin{align*}
[M\upk]_{k_1,k_2} 
&= \int_{[0,1]^d} e_{k_1}(x)e_{k_2}(x)^\ast e^{-2\pi i x^\top k} dx \\
&= \int_{[0,1]^d} e^{-2\pi i x^\top (k + k_1 - k_2)} dx
= \indicator{k = k_2 - k_1}.
\end{align*}
Therefore, the entries of $M\upk$ are bounded by $1$, and if $\abs{k} \geq 2t$, then $M\upk = 0$. Therefore, we can bound $\nrm{M\upk}_{Frob.} \leq n\indicator{\abs{k}\leq 2t}$. Since $\#\{\abs{k}\leq 2t\} \leq (4t+1)^d \leq 8^d t^d$ when $t \geq 1$, then
\[
\sum_{k\in\mathbb{Z}^d} \nrm{M\upk}_{Frob.}  \leq n 8^{2d} t^{2d}  .
\]

\section{Experimental Details}\label{app:experimental-details}

Here, we describe how we applied our approach to two simple non-convex optimization problems to show its promise. As described, we can lower bound $c_*$ by solving the stochastic concave maximization problem \cref{eq:stoch-opt-problem}
using an algorithm like projected stochastic gradient ascent. However, each iteration requires projecting the algorithm's iterate onto the PSD cone, which is computationally expensive. 

Therefore, following \citet{burer2003nonlinear}, we reparametrize $A = UU^\ast$, which is always positive semidefinite, using new parameters $U\in\C^{n\times n}$, yielding the \emph{unconstrained} objective
\begin{equation}\label{eq:factorized-objective}
\bar{c} = \max_{U \in \C^{n\times n}} \E_{k\sim\pi}\brk*{L_k(UU^\ast)}.
\end{equation}
Due to the non-linear reparametrization, the objective is no longer concave, but just as \citet{burer2003nonlinear} exhibit for linear SDPs, we find that stochastic gradient ascent on $U$ succeeds for our problem when we optimize a smooth surrogate for our non-differentiable objective. Specifically, for $k \neq 0$, we replace 
\[
L_k(A) = \frac{-1}{\pi_k}\abs*{\hat{f}_k - \inner{A}{M\upk}}
\to \tilde{L}_k(A) = \frac{-1}{\pi_k}\sqrt{(\alpha\pi_k)^2 + \abs*{ \hat{f}_k - \inner{A}{M\upk}}^2},
\]
where $\alpha$ is small scalar. Choosing $\alpha$ larger makes the objective smoother, but makes $\tilde{L}_k$ a worse approximation of $L_k$. In our experiments, we tune $\alpha$, along with the other hyperparameters---including the stepsize, $\eta$ and the number of iterations, $T$---with cross validation.

\paragraph{A random non-convex objective.}
We constructed a family of non-convex periodic functions on $[0,1]$ and $[0,1]^2$ to test our algorithm. The functions are defined in terms of their Fourier series, with $\hat{f}_k \sim \mc{N}(0,1/(1+\abs{k})^2) + i\cdot\mc{N}(0,1/(1+\abs{k})^2)$ for each $k$ with $\abs{k} \leq 15$ in the 1D case and $\abs{k} \leq 4$ in the 2D case. We then adjust the Fourier components so that they satisfying the necessary property  $\hat{f}_{k^\ast} = \hat{f}_k^\ast$. The value of $f$ itself is then computed on a grid of points, $x_1,\dots,x_N$, and is rescaled by dividing by $\max_i f(x_i) - \min_j f(x_j)$ so that $f$'s range is of order 1. 

\paragraph{A different feature map}
The feature map that we use for our experiments has the form $\phi_{n,\rho}(x) = \tilde{\phi}_{n,\rho}(x[1])\circ\tilde{\phi}_{n,\rho}(x[2])\circ\dots\circ\tilde{\phi}_{n,\rho}(x[d])$, where $n \in \mathbb{N}$ and $\rho\in(0,1)$ are hyperparameters to be chosen later, $\tilde{\phi}_{n,\rho}:\R\to\C^n$, and $\circ$ denotes the hadamard product, so the feature map decomposes over the coordinates of $x$. To define $\tilde{\phi}_{n,\rho}$, we sample $n$ points $x_1,\dots,x_n$ uniformly at random from $[0,1]^d$, and set
\[
\tilde{\phi}_{n,\rho}(x[i])[j] = \varphi_{\rho}(x[i] - x_j[i]),\qquad \varphi_\rho(x) = \sum_{k\in\mathbb{Z}}\rho^{\abs{k}}e^{2\pi i k x}.
\]
The function $\varphi$ is chosen so that its Fourier components $\hat{\varphi}_k = \rho^{\abs{k}}$ decay exponentially quickly with $k$. In Appendix \ref{app:experiment-feature-map} below, we show how to compute the matrices $M\upk$ that are needed to implement our algorithm, and we bound $\nrm{M\upk}_{Frob.}$, which is needed to compute the a posteriori guarantees. For this feature map, larger $n$ allows for a more expressive, but more computationally expensive PSD model and Figures \ref{fig:empirical-results1d} and \ref{fig:empirical-results2d} demonstrate the effect of $n$ on our a posteriori accuracy guarantees in one and two dimensions, respectively. The parameter $\rho$, which we choose using cross-validation, clearly affects the Fourier components of the PSD model that we learn, with smaller $\rho$ making them decay more quickly with $k$.

\subsection{Analysis of the Feature Map}\label{app:experiment-feature-map}

In what follows, we will drop the subscripts $n$ and $\rho$ and consider these hyperparameters to be fixed and arbitrary. We recall the definition 
\[
\tilde{\phi}(x[i])[j] = \varphi(x[i] - x_j[i]),\qquad \varphi(x) = \sum_{k\in\mathbb{Z}}\rho^{\abs{k}}e^{2\pi i k x}.
\]
We further note that 
\begin{align}
\varphi(x) &=  \sum_{k \in \Z} \rho^{|k|} e^{2i\pi k x} 
= -1 + 2 \cdot {\rm Re} \Big( \sum_{k \in \mathbb{N}} \rho^k  e^{2i\pi k x}  \Big)
\nonumber\\
& =  -1 + 2 \cdot {\rm Re} \Big( \frac{1}{1- \rho  e^{2i\pi x}} \Big)
= - 1 + 2 \frac{ 1 - \rho \cos 2\pi x}{ 1 + \rho^2 - 2 \rho \cos 2 \pi x} \nonumber\\
& =  
\frac{1-\rho^2}{1 + \rho^2 - 2 \rho \cos 2 \pi x}.\label{eq:new-kernel-positive}
\end{align}

In this Appendix, we show how to compute $M\upk$, the $k$th Fourier component of $\phi\phi^\ast$, which is needed to implement our algorithm. Since $\phi(x) = \tilde{\phi}(x[1])\circ\dots\circ\tilde{\phi}(x[d])$ decomposes across coordinates, this essentially boils down to computing the 1D version $d$ times and multiplying across dimensions.

So, for now we focus on the case $d=1$, and attempt to compute
\[
[M\upk]_{ij} = [\widehat{\phi_i\phi_j}]_k.
\]
Since $\phi_i(x) = \phi(x - x_i)$ and $\phi_j(x) = \phi(x - x_j)$, we need to know how to compute the $k$-th Fourier coefficient of $x \mapsto \varphi(x-y)\varphi(x - z)$. We have:
\BEAS
\varphi(x-y)\varphi(x - z)
& = & \sum_{n,m \in \Z} \rho^{|n| + |m| } e^{2i\pi n (x-y) + 2i\pi m (x-z) }
\\
& = &  \sum_{n,m \in \Z} \rho^{|n| + |m| } e^{-2i\pi n y -  2i\pi m z }  e^{2i\pi (n+m) x }.
\EEAS

For simplification and by symmetry, we can consider $T(y,x) = \varphi(x-y)\varphi(x)$, so that 
$\phi(x-y)\phi(x - z) = T(y-z,x-z)$.
Thus, this $k$-th Fourier coefficient is simply
\BEAS
\hat{T}(y)_k & = & \sum_{n+m = k}  \rho^{|n| + |m| } e^{-2i\pi n y  }  =    \sum_{n \in \Z}  \rho^{|n| + |n-k| } e^{-2i\pi ny}  .
\EEAS
Moreover, we will need to compute $e^{-2ik\pi z}  \hat{T}(y-z)_k$.
We directly have $\hat{T}(y)_{-k} = \hat{T}(y)_k^\ast$, so we can consider $k \geq 0$ and
\BEAS
\hat{T}(y)_k& = &  
\sum_{n = -\infty }^{0}  \rho^{k - 2n  } e^{-2i\pi n  y}  
+ \sum_{n = 1}^{k-1}  \rho^{k  } e^{-2i\pi n y}  
+\sum_{n = k}^{+\infty}  \rho^{2n - k } e^{-2i\pi n y}  
\\
 & = &  
\sum_{n = 0 }^{+\infty}  \rho^{k + 2n  } e^{2i\pi n y}  
+  \rho^{k  }  \sum_{n = 1}^{k-1} e^{-2i\pi n y}  
+\sum_{n = k}^{+\infty}  \rho^{2n - k } e^{-2i\pi n y}  
\\
 & = &  \begin{cases} \rho^{k}\prn*{\frac{1  }{1- \rho^2 e^{2i\pi y}}
 +   
  \frac{ e^{-2i\pi  y} -   e^{-2i\pi k y}}{1-   e^{-2i\pi y}}
+  \frac{ e^{-2i\pi ky}  }{1- \rho^2 e^{-2i\pi y}}} & y \neq 0 \\
\rho^{k}\prn*{
  k
+  \frac{ 1+\rho^2 }{1- \rho^2 }} & y = 0.
\end{cases}
\EEAS
Therefore,
\BEAS
e^{-2ik\pi z} \hat{T}(y-z)_k& = & \begin{cases} \rho^{k}\prn*{  \frac{e^{-2ik\pi z}  }{1- \rho^2 e^{2i\pi (y-z) }}
+
\frac{ e^{-2i\pi (y + kz)} -   e^{-2i\pi (k y+z)} }{e^{-2i\pi z }-   e^{-2i\pi y }}
+ \frac{ e^{-2i\pi k y}  }{1- \rho^2 e^{-2i\pi (y-z)}}
 } & y \neq z\\
 \rho^{k}e^{-2ik\pi z}\prn*{
  k
+  \frac{1 + \rho^2}{1- \rho^2}}
 & y = z.
\end{cases}
\EEAS
Therefore, we have that the $i,j$th entry of $M\upk$ for $k \geq 0$ is given by
\[
[M\upk]_{ij} = \begin{cases}
\rho^{k}\prn*{  \frac{e^{-2ik\pi x_j}  }{1- \rho^2 e^{2i\pi (x_i-x_j) }}
+
\frac{ e^{-2i\pi (x_i + kx_j)} -   e^{-2i\pi (k x_i+x_j)} }{e^{-2i\pi x_j }-   e^{-2i\pi x_i }}
+ \frac{ e^{-2i\pi k x_i}  }{1- \rho^2 e^{-2i\pi (x_i-x_j)}}
 } & x_i \neq x_j \\
 \rho^{k}e^{-2i\pi x_i}\prn*{k
+  \frac{1 + \rho^2}{1- \rho^2}} & x_i = x_j,
\end{cases}
\]
and for $k < 0$, we have $M^{(k)} = {M^{(-k)}}^\ast$.
This allows us to compute $[M\upk]_{ij}$ in the 1D case, which is the above function of $k$, $x_i$, and $x_j$; denote this function $h(k,x_i,x_j)$. For the multidimensional case, since the feature map decomposes over coordinates, we simply have
\[
[M\upk]_{ij} = \prod_{a=1}^d h(k, x_i[a], x_j[a]).
\]
With this in hand, we can implement our algorithm.


\paragraph{Special cases and bounds.}
Now, we try to control $\nrm{M\upk}_{Frob.}$, which is needed to compute the a posteriori error guarantees.

First, we note that in the special case $k=0$, we get:
$$
[M^{(0)}]_{ij} = 
  \frac{ 1-\rho^4}{1+\rho^4 - 2 \rho^2 \cos 2\pi (x_i-x_j)}
 .
$$
Moreover, we have:
$$
\hat{T}(0)_k = \rho^{|k|} 
\Big[   |k| 
  + 
  \frac{1 + \rho^2 }{1 -  \rho^2  } \Big].
$$
We also have, since $\varphi$ is always non-negative (see \eqref{eq:new-kernel-positive}):
$$
|\hat{T}(y)_k| = \Big| \int_0^1 e^{-2ik\pi x}  \varphi(x)\varphi(x-y) dx  \Big| \leq \int_0^1 \varphi(x)\varphi(x-y) dx  =  \hat{T}(0)_k.
$$
Therefore, in 1D
\[
\abs{[M\upk]_{ij}} = \abs{e^{-2i\pi kx_j}\hat{T}(x_i-x_j)_k} \leq \abs{\hat{T}(x_i-x_j)_k} \leq \hat{T}(0)_k = \rho^{|k|} 
\Big[   |k| 
  + 
  \frac{1 + \rho^2 }{1 -  \rho^2  } \Big]
\]
Therefore, we can upper bound in 1D
\[
\nrm{M\upk}_{Frob.}
= \sqrt{\sum_{i,j=1}^n [M\upk]_{ij}^2} 
\leq n \rho^{|k|} 
\Big[|k| 
  + 
  \frac{1 + \rho^2 }{1 -  \rho^2  } \Big].
\]
In multiple dimensions, we can further bound
\begin{align*}
\nrm{M\upk}_{Frob.}
&\leq n^d \prod_{i=1}^d \rho^{|k_i|} 
\Big[|k_i| 
  + 
  \frac{1 + \rho^2 }{1 -  \rho^2  } \Big]\\ 
&\leq 
n^d\rho^{\abs{k}}\brk*{\frac{\abs{k}}{d} + \frac{1+\rho^2}{1-\rho^2}}^d \\
&= \prn*{n\frac{1-\rho^2}{1+\rho^2}}^d\rho^{\abs{k}}\brk*{1 + \frac{\abs{k}\frac{1-\rho^2}{1+\rho^2}}{d}}^d \\
&\leq \prn*{n\frac{1-\rho^2}{1+\rho^2}}^d\rho^{\abs{k}}e^{\abs{k}\frac{1-\rho^2}{1+\rho^2}} \\
&= \prn*{n\frac{1-\rho^2}{1+\rho^2}}^d\prn*{\rho e^{\frac{1-\rho^2}{1+\rho^2}}}^{\abs{k}} \\
&= \zeta {\tilde{\rho}}^{\abs{k}},
\end{align*}
where $\zeta$ is a constant independent of $k$ and $\tilde{\rho} < 1$. Therefore, $\nrm{M\upk}_{Frob.}$ decays exponentially quickly as $\abs{k}$ increases, which ensures that $\sum_{k\in\mathbb{Z}^d} \nrm{M\upk}_{Frob.}$ is finite and not too large, and that
$\sum_{k:\abs{k} > K} \nrm{M\upk}_{Frob.}$ goes to zero as $K$ increases, which can allow for a tight a posteriori guarantee using Theorem \ref{thm:a-posteriori-accuracy}.

\end{document}